\documentclass[a4paper,11pt]{amsart}
\usepackage{amsthm,amsmath}
\usepackage{graphicx}
\usepackage{latexsym,amssymb}
\usepackage{todonotes}
\usepackage{enumerate}
\usepackage{url}

\theoremstyle{plain}
\newtheorem{lem}{Lemma}

\newtheorem{thm}[lem]{Theorem}

\theoremstyle{definition}
\newtheorem{defn}{Definition}
\newtheorem{example}{Example}
\theoremstyle{remark}
\newtheorem{rem}[lem]{Remark}

\newcommand{\R}{\mathbb{R}}
\newcommand{\Q}{\mathbb{Q}}
\newcommand{\Z}{\mathbb{Z}}
\newcommand{\C}{\mathbb{C}}
\def\P{\mathbb P}
\def\H{\mathbb H}
\newcommand{\D}{\mathbb{D}}

\newcommand{\ci}{\mathrm{i}}
\newcommand{\qi}{\mathbf{i}}
\newcommand{\qj}{\mathbf{j}}
\newcommand{\qk}{\mathbf{k}}
\newcommand{\eps}{\epsilon}
\DeclareMathOperator{\norm}{N}
\DeclareMathOperator{\arccot}{arccot}
\DeclareMathOperator{\rank}{rank}

\newcommand{\ii}{\mathbf{i}}
\newcommand{\jj}{\mathbf{j}}
\newcommand{\kk}{\mathbf{k}}

\title{An Algebraic Study of Linkages with Helical Joints}
\author{Hamid Ahmadinezhad, Zijia Li, and Josef Schicho}

\begin{document}

\begin{abstract} Methods from algebra and algebraic geometry have been used in various ways to study 
  linkages in kinematics. These methods have failed so far for the study of 
 linkages with helical joints (joints with screw motion), because of the presence of some non-algebraic 
 relations. In this article, we explore a delicate reduction of some analytic equations in kinematics
to algebraic questions via a theorem of Ax. As an application, we give a classification of mobile closed 
$5$-linkages with revolute, prismatic, and helical joints.
\end{abstract}

\maketitle

\setcounter{tocdepth}{1}
\tableofcontents

\section{Introduction}

Linkages, and in particular closed linkages, are a crucial object of study in the modern theory of kinematics.
The use of algebra and geometry for studying linkages is very classical and goes back to Sylvester, Kempe,
Cayley and Chebyshev.

A linkage, as appearing in robotics/mechanical engineering, biology, as well as modelling of molecules in chemistry, etc.,
is a mechanical structure that consists of a finite number
of rigid bodies -- its {\em links} -- and a finite number of {\em joints} that connect the links together, 
so that they possibly produce a motion.
A linkage is called {\em closed} if its number of links and joints are equal and they are connected cyclically.
%
%
%
We consider four types of joints:
\begin{enumerate}[]
\item (R) revolute joints: allow rotations around a fixed axes; 
\item (P) prismatic joints: allow translations in a fixed direction;
\item (C) cylindrical joints: allow rotations around a fixed axes and translations in the
the direction of the axes;
\item (H) helical joints: allow the motions of a cylindrical joint where the rotation
angle and the translation length are coupled by a linear equation.
\end{enumerate}

We will use the notation R-joint for a revolute joint, and similarly for other types.
Note that the dimensions of the set of allowed motions (the degree of freedom) is $1$ for joints of type R, P, and H, and
$2$ for C-joints.

The {\em configuration set} of a closed linkage $L$, denoted by $K_L$, is the set of possible simultaneous motions of all joints
(see Definition~\ref{conf} for a precise description).
The dimension of $K_L$ is called the {\em mobility} of $L$, and $L$ is called {\em mobile}
if the mobility is positive.

It is known that $K_L$ can be described by analytic equations; see \cite{seligbook}, page 356.
If there are no H-joints, then we also have
a description by algebraic equations. We refer to \cite{history} for a historic overview 
of the use of geometric algebra in kinematics. This is the subject that has attracted algebraists the most. 
It should be mentioned, however, that also number theory has been used for studying linkages:
in \cite{gs}, reduction modulo prime numbers are considered in order to construct a new family
of Stewart-Gough platforms.
Below, in Section~\ref{sec:bonds} 
we briefly explain the algebraic setup, as well as the theory of bonds, a rather new combinatorial technique 
that has shown to be very useful for analysing closed linkages with R-joints \cite{hss2}. 
We also explain the analytical relations in the presence of H-joints. There are also other (numerical) 
algebraic methods that are applied in kinematics, see for example \cite{bpsc, res, ws}.

A closed linkage with $n$ joints, where all joints are R-joints, is denoted by $n$R-linkage. 
We denote by $n$-linkage a linkage with $n$ joints where no information on the type of joints is specified. 
It is easy to imagine that a $3$R-linkage does not have a motion, and hence its configuration set is trivial. 
On the other hand a generic $n$R-linkage for $n\geq 7$ has positive mobility (see \cite{seligbook}, page 356), and
hence there is not much to study. So the interesting cases are when $n=4,5$ or $6$. Nowadays we have a full 
classification of $4$R- and $5$R-linkages with mobility one ($\dim K_L=1$), 
and we know many cases for $n=6$ (see \cite{dietmaier}). It is an open research problem 
to classify all $6$R-linkages. These classification problems are considered by algebraists. 
Because of the nature of other types of linkages it seemed difficult, or rather impossible, to be able to use 
any of the present algebraic techniques for linkages with H-joints.

\subsubsection*{\bf What is new in this article?}
As a main result, we show that unexpected mobility of a linkage with H-joints, i.e., a mobility
that is strictly bigger than predicted by the Gr\"ubler-Kutzbach-Chebychev formula which simply
counts parameters and equational restrictions, can always be explained algebraically. 
Let $L$ be a linkage with H-joints, and let $L^\prime$ be the linkage 
obtained from $L$ by replacing all H-joints by C-joints. It is clear that the configuration set of $L$ 
is a subset of the configuration set of $L^\prime$.
The relation between the configuration set of $L^\prime$ and the Zariski closure of the configuration
set of $L$ will be made very precise (Theorem~\ref{thm:ax}), with the help of Ax's Theorem~\cite{ax}
on the transcendence degree of function fields with exponentials.  
Note that Ax's theorem is originally about Schanuel's conjecture in number theory and has no apparent connections to kinematics.

The mobile linkages with 4 joints of type H, R, or P have been
classified in \cite{Delassus}. Here (more precisely in Theorem~\ref{thm:h5}) 
we give a classification of mobile linkages with 5 joints of type H, R, or P.
Using our main result, we reduce to linkages with joints of type R or P only. The classification of mobile 5R-linkages
has been done in \cite{Karger}, but for linkages with both R-joints and P-joints, we could not find a complete
classification in the literature. On the other hand, this classification is not difficult when we use the theory
of bonds, so we also give it in here (Theorem~\ref{thm:p}). 

\subsubsection*{\bf Structure of the paper}
In Section \ref{sec:bonds} we setup a mathematical language to describe and to analyse linkages with arbitrary types
of joints, we recall the theory of bonds for R-joints and we introduce the adaptions of this theory that make it
work also for P-joints. In Section \ref{non-helical} 
 we use these algebraic methods to classify mobility
for 5-linkages with P- and R-joints. In Section \ref{construction-helical} we reduce H-joints to C-joints.
Finally, in Section \ref{classification-helical} we use Ax's theore and the results of the previous sections
to classify mobile 5-linkages with helical joints.

Because the two proofs of the classification results, i.e., Theorem~\ref{thm:p} and Theorem~\ref{thm:h5}, are methodically
quite different, it is possible to apply a ``filter'' while reading in case one is only interested in the main result
(or in one of the proofs using Ax's theorem). In that case, the reader could omit the second half of Section \ref{sec:bonds}, right before bonds
are introduced, and the whole Section \ref{non-helical}, except for Theorem~\ref{thm:p} which has to be taken for
granted.

\section{Algebraic set up} \label{sec:bonds}

In this section we set up the notation for an algebraic description of linkages with arbitrary joints.
Then we briefly recall bonds for R-joints as defined in \cite{hss2}. Finally, we introduce bonds for P-joints
and prove some basic properties of them.

\subsection*{Dual quaternions and configuration set}

Suppose $\R$ is the set or real numbers, $\D:=\R+\eps\R$ is the ring of dual numbers and $\eps^2=0$. 
Denote by $\H$ the non-commutative algebra of quaternions, where 
\[\H=\{A=a_0+a_1\ii +a_2\jj +a_3\kk \text{ where } \ii^2\!=\jj^2\!=\kk^2\!=-1\text{ and }\ii.\jj\!
 =\!\kk\,, \jj.\kk\!=\!\ii\,, \kk.\ii\!=\!\jj\}.\]
Let $\D\H:=\D\otimes_\R\H$ denote the dual quaternions, i.e.
\[\D\H=\{h=A+\eps B\text{ where } A, B\in\H\text{ and }\eps^2=0\}.\]
We call $A$ the {\em primal} part of a dual quaternion 
 $h$, and $B$ the {\em dual} part of $h$. 
The conjugate of $A\in\H$, as above, is defined by $\overline{A}= a_0-a_1\ii -a_2\jj -a_3\kk$. 
This extends naturally to define the conjugate dual quaternion
of $h$ by
\[\overline{h}=\overline{A}+\eps\overline{B}\]
The norm function $\norm\colon\D\H\rightarrow\D$ is then defined by 
$\norm(h) = h.\overline{h}=A\overline{A}+\eps(A\overline{B}+\overline{A}B)$;
  and the latter is called the \emph{norm} of $h$.

Note that $\D\H$ can be regarded as a real $8$-dimensional vector space, and projectivising $\D\H$ we
obtain $\P^7$. 
The {\em Study quadric} $S$ is a hypersurface of this projective space defined by the quadratic equation 
\[\sum_{i=0}^3 a_ib_i=0\]
where $h=a_0+a_1\ii +a_2\jj +a_3\kk+\eps(b_0+b_1\ii +b_2\jj +b_3\kk)$. In other words
\[S=\{h\in\P^7\text{ such that }\norm(h)\in\R\}\]

The linear $3$-space represented by all dual quaternions with zero primal part is denoted
by $E$. It is contained in the Study quadric, and the complement $S-E$ is
closed under multiplication and multiplicative inverse; hence $S-E$ 
forms a group, which is isomorphic to the group of Euclidean displacements (see \cite[Section~2.4]{husty10}).

For a natural number $n\in\mathbb{N}$, 
a {\em linkage} with $n$ joints is described as an $n$-tuple $L=(j_1,\dots,j_n)$, where each $j_i$ represents a joint. 
We will use cyclic notation for joint indices, i.e., $j_{n+1}=j_1$.

\subsection*{Quantisation}
The type of joint specifies which data must be given in order to determine the set of possible motions, as follows. Suppose $k\in\{1,\dots,n\}$.
\begin{enumerate}[I.]
\item If $j_k$ is an {\em R-joint}, then we specify a dual quaternion $h_k$ such that $h_k^2=-1$.
We write $p_k$ and $q_k$ for the primal and the dual part of $h_k$, i.e., $h_k=p_k+\eps q_k$
with $p_k,q_k\in\H$. 
The set of possible motions is parametrized by the joint parameter $t_k\in\P^1$,
which is the rotation with an angle of $2\arccot(t_k)$. The rotation corresponds
to the dual quaternion $m_k=t_k-h_k$, and to $1$ if $t_k=\infty$. 
Note that the latter means we have fixed the initial position at $\infty$.

\item If $j_k$ is a P-joint, then we specify a quaternion $p_k\in\H$ such that $p_k^2=-1$.
The set of possible motions is parametrized by the joint parameter $s_k\in\R$,
and the translation corresponds to the dual quaternion $m_k=1-\eps s_kp_k$. 

\item If $j_k$ is a {\em C-joint}, then we specify a dual quaternion $h_k=p_k+\eps q_k$ such that $h_k^2=-1$,
The set of possible motions is parametrized by the joint parameters $(s_k,t_k)$ and 
corresponds to the dual quaternion $m_k=(1-\eps s_kp_k)(t_k-h_k)$.

\item If $j_k$ is an {\em H-joint}, then we specify a dual quaternion $h_k=p_k+\eps q_k$ such that $h_k^2=-1$,
and a nonzero real number $g_k$. The number $\frac{g_k}{2\pi}$ is often refereed to as the {\em pitch} in mechanical engineering. 
The joint parameter is $\alpha_k\in\R$, and the motion corresponds to the dual quaternion
$m_k=(1-\eps g_k\alpha_k p_k)(1-\tan(\frac{\alpha_k}{2})h_k)$.

\end{enumerate}

The data which must be specified for all joints are called the {\em geometric parameters}. Note
that when the linkage moves, the geometric parameters also change. However, there are functions in 
the geometric parameters that do not change when the linkage moves, such as the normal distance and the
angle between neighbour rotation or helical axes.

\begin{defn}\label{conf}The {\em configuration set} $K$ is the set of all parameters $t_k,s_k,\alpha_k$ such that the closure equation
\begin{equation} \label{eq:clos}
 m_1 m_2\cdots m_n \equiv 1 
\end{equation}
is fulfilled. The symbol $\equiv$ stands for projective equivalence, i.e., up to multiplication by a nonzero real scalar.

The {\em mobility of $K$} is the dimension of the solution set of Equation~\ref{eq:clos} as a complex analytic
set in the parameter space. If it is positive, then we say that $L$ is {\em mobile}.
\end{defn}

We are interested in mobile linkages, and mainly those with mobility~$1$. Finding such linkage for
given types, and numbers, of joints is a main goal. This leads to analysing the solutions of Equation~(\ref{eq:clos}).

If all values for $t$-parameters are $\infty$ and all values for $s$- and $\alpha$-parameters are $0$,
then all $m_k$ are equal to 1 and Equation~(\ref{eq:clos}) is fulfilled. This point of $K$ is called the
{\em initial configuration} of~$L$. 

\begin{rem}
The dimension of $K$ as a real analytic set would be a more interesting number than the mobility we defined above,
but it is harder to control. For instance, planar $4$R-linkages always have mobility~$1$, but the real dimension
can also be $0$. In any case, the complex dimension is an upper bound for the real dimension, and if the two numbers
are not equal then all real configurations must be singularities of the complex configuration space.
\end{rem}

The remaining part of this section is concerned with the theory of bonds. If the reader is willing to
believe Theorem~\ref{thm:p}, he/she may jump forward to this theorem, skip its proof and proceed
with Section~\ref{construction-helical}, in order to get faster to the application of Ax's theorem.

\subsection*{Bonds}

For linkages with R-joints, bonds have been introduced in \cite{hss2} as an algebraic tool
which is used to describe and understand the algebraic structure of the configuration set.
Informally speaking, bonds are the points in the boundary of the compactification of the complex
configuration set. The closure equation degenerates in the boundary, and one obtains useful
algebraic consequences. At this moment, we do not have any geometric intuition for bonds,
we just use them mainly as a tool for studying the (real) configuration set and for 
deriving geometric conditions of the rotation axes.
Here is the precise definition.

\begin{defn} Let $L$ be an $n$-linkage with joints of type R or P.
Suppose $Z$ is the projective closure of the complexification of $K_L$ in $(\P^1_\C)^{n}$.
Then the {\em bond set} $B$ is defined as the intersection of $Z$ and the solution set of the bond equation
\begin{equation} \label{eq:bond}
 m_1 m_2\cdots m_n = 0 .
\end{equation}
\end{defn}

\begin{rem}
In \cite{hss2}, where the theory of bonds is initially developed, only linkages with mobility $1$ are considered 
and bonds are defined as points on the normalisation of the curve $K$.
In this paper, however, this is not necessary because we do not need multiplicities of bonds. Hence we can afford to
simply say a bond is a point of $B$.
\end{rem}


\subsubsection*{\bf Construction of the bond diagram}
Let $\beta$ be a bond. We say that $\beta$ is {\em attached} to a joint $m_k$ if $\norm(m_k(\beta))=0$.
This is equivalent to $t_k^2+1=0$ if $j_k$ is an R-joint, and to $s_k=\infty$ if $j_k$ is a P-joint.
If $\beta$ is attached to two different joints $j_k$ and $j_\ell$, then we say that $\beta$ {\em connects} 
$j_k$ and $j_\ell$ if and only if
\[ m_k(\beta) m_{k+1}(\beta) \cdots m_\ell(\beta) = m_\ell(\beta) m_{\ell +1}(\beta) \cdots m_k(\beta) = 0 \]
(this definition is slightly different from the definition in \cite{hss2}, but the more complicated definition
using multiplicities is not needed here).

\begin{defn} Suppose $G=(V,E)$ is the graph of a linkage, where vertices, elements of $V$, 
represent the rigid bodies of $L$ and $v_i,v_i\in V$ are connected via an edge $e\in E$ is there is a 
joint between them. The {\em bond diagram} is then defined to be this graph together with 
the following extra information: two edges are connected if their corresponding joints 
are connected via a bond.\end{defn}

\begin{example}
There is a unique family of 4R linkages such that the four axes are not all parallel and do not all have
a common point, the Bennett linkage (see \cite{bennett}). Its configuration curve can be defined by the equations
\[ t_1=t_3, t_2=t_4, at_1s-t_2+b=0 , \]
where $a,b\in\R$, $(a,b)\ne (1,0)$, $a\ne 0$ are parameters. The bond set is 
\[ B = \{ (\pm\ci,a\pm\ci+b,\pm\ci,a\pm\ci+b), ((\pm\ci-b)/a,\pm\ci,(\pm\ci-b)/a,\pm\ci)\} , \]
and the bond diagram is shown in Figure~\ref{fig:bonds}.
\end{example}

We define the offset $o(h_1,h_2,h_3)$ of three lines $h_1,h_2,h_3$ as follows. We assume
that neither $h_1||h_2$ nor $h_2||h_3$, where the symbol $||$ is used to show the two lines are parallel (otherwise, the
offset is not defined). Let $n_{12}$ be the common normal of $h_1,h_2$, i.e., the unique line intersecting both
$h_1$ and $h_2$ at a right angle. Let $n_{23}$ be the common normal of $h_2,h_3$. 
Then $o(h_1,h_2,h_3)$ is defined as the signed distance between the intersection of $h_2,n_{12}$ 
and the intersection of $h_2,n_{23}$. 
The sign comes from the orientation of the line $h_2$ represented as a dual quaternion.
The offset of three consecutive R-joints is fixed when the linkage moves.

\begin{example}
Assume that the lines $h_1,h_2,h_3$ are coplanar and pairwise not parallel. Then $o(h_1,h_2,h_3)$ is the 
distance of the intersection points $h_1\cap h_2$ and $h_2\cap h_3$. If $h_3$ is rotated around $h_2$ by
an angle different from $\pi$, call the result $h_3'$, then $h_1,h_2,h_3'$ will not be coplanar,
but we still have $o(h_1,h_2,h_3)=o(h_1,h_2,h_3')$.
\end{example}

We recall some well-known facts on the bond diagram, and refer to \cite{hss2} for details. 

\begin{enumerate}[(i)]
\item Every bond is attached to at least two joints.
\item If a bond is attached to a joint $j_k$, then it connects $j_k$ to at least one other joint.
\item If a joint $j_k$ actually moves during the motion, then it is attached to at least one bond. 
\item Two consecutive R-joints, $j_i$ and $j_{i+1}$, are not connected by a bond.
\item If $j_i,j_{i+1}$ and $j_{i+2}$ are R-joints with axes $h_i,h_{i+1}$ and $h_{i+2}$ such that 
  $j_i$ is connected to $j_{i+2}$ and $h_1||h_2$, then $h_2||h_3$.
\item If $j_i,j_{i+1}$ and $j_{i+2}$ are R-joints with axes $h_i,h_{i+1}$ and $h_{i+2}$ such that $h_i$ 
  is not parallel to $h_{i+1}$, and $j_i$ is connected to $j_{i+2}$, then $o(h_i,h_{i+1},h_{i+2})=0$. 
\end{enumerate}

\section{5-linkages with Revolute and Prismatic Joints}\label{non-helical}

In this section we classify mobile closed 5-linkages with R- and P-joints. For the case of R-joints only,
this is well-known, as described below. The general case is handled by bond theory; this makes the proof quite conceptual and
avoids long and technical calculations.
We use the results of this section to classify 5-linkages with helical joints in Section~\ref{classification-helical}.

If $L$ has two neighbouring R-joints with equal axes or two neighbouring P-joints with equal directions,
then we say that $L$ is {\em degenerate}. 
Throughout this section, we assume that $n=5$ or $n=4$, and $L=(j_1,\dots,j_n)$ is a mobile linkage 
with configuration set $K$.
We also assume that $L$ is not degenerate, and that no joint parameters are constant during motion of the linkage
(otherwise one could easily make $n$ smaller).

If $L$ has only R-joints, then we have one of the following three cases (\cite{Karger}; see \cite{hss2} for
a proof using bond theory).

\begin{enumerate}
\item $L$ is spherical, i.e., all rotation axes meet in the same point; then $L$ has mobility 2.
\item $L$ is planar, i.e., all rotation axes are parallel; then $L$ has mobility 2.
\item $L$ is a Goldberg linkage, constructed as follows: take two spatial 4-linkages with one joint and
	one link in common; then remove the common link. The mobility of the Goldberg linkage is 1.
	If $h_1,\dots,h_5$ are the rotation axes,
	then $$o(h_4,h_5,h_1)=o(h_5,h_1,h_2)=o(h_1,h_2,h_3)=0,\text{ and }o(h_2,h_3,h_4)=\pm o(h_3,h_4,h_5)$$
	up to cyclic permutation of joints.
\end{enumerate}

\begin{figure}[htb]
  \centering
  \includegraphics{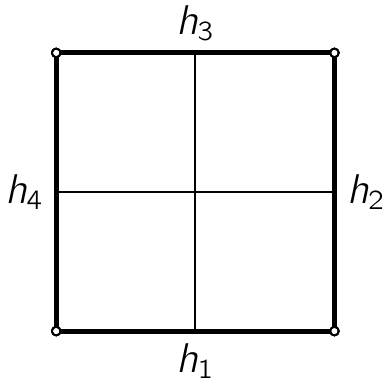}
  \hspace{1.6cm}
  \includegraphics{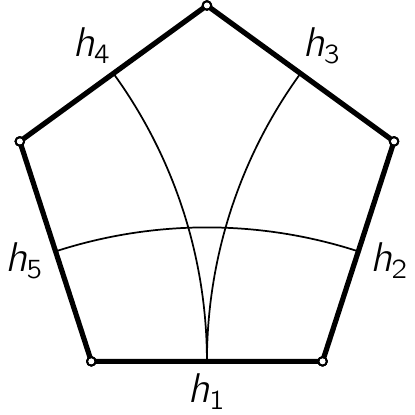}
  \caption{Bond diagrams for the Bennett 4R and the Goldberg 5R linkage. The vanishing of all offsets
	of the Bennett 4R linkage and of the three offsets $o(h_4,h_5,h_1)$,  $o(h_5,h_1,h_2)$, and $o(h_1,h_2,h_3)$
	of the Goldberg 5R linkage are easy consequences of bond theory.}
  \label{fig:bonds}
\end{figure}

By considering the specified data and parametrized motions modulo $\eps$, we may construct a spherical linkage $L'$,
the {\em spherical projection} of $L$. The P-joints of $L$ disappear, their translation motions are projected
to the identity. Parallel joint axes of $L$ are projected to identical axes of $L'$. Note that $K$ is projected
to an algebraic subset of $K'$ (the configuration set of $L'$). This subset has positive dimension if and only
if $L$ has at least one R-joint.

\begin{lem} \label{lem:3p}
If $L$ has $2$ or more P-joints, then all rotation axes are parallel.
\end{lem}

\begin{proof}
Let $r\ge 2$ be the number of P-joints. Then $L'$ is a spherical linkage with $5-r\le 3$ joints. Such a linkage
is necessarily degenerate: if all three (or fewer) joints are actually moving, then all axes are identical.
\end{proof}

In order to classify PRRRR linkages, we use bond theory. 
Here are some additional facts for bonds in the presence of  P-joints.

\begin{lem} \label{lem:1}
Assume that $j_i$ is a P-joint, and $j_{i+1}$ and $j_{i+2}$ are R-joints.
\begin{enumerate}[(a)]
\item The joints $j_i$ and $j_{i+1}$ cannot be connected by a bond.
\item If the joints $j_i$ and $j_{i+2}$ are connected by a bond, then the axes $h_{i+1}$ and $h_{i+2}$ are parallel.
\end{enumerate}
\end{lem}

\begin{proof}
\emph{(a)} Assume, without loss of generality, that $(\infty,t)$ are the coordinates at $(j_i,j_{i+1})$ of the bond connecting $j_i$ and $j_{i+1}$. 
Then we have that $\eps p_i(t-h_{i+1})=0$, which is impossible.

\emph{(b)} If $(\infty,t,t')$ are the coordinates at $(j_i,j_{i+1},j_{i+2})$ of the bond connecting $j_i$ and $j_{i+1}$, then 
$\eps p_i(t-h_{i+1})(t'-h_{i+2})=0$. We pass to the dual part. Since $p_i$ is invertible,
we conclude that $(t-p_{i+1})(t'-p_{i+2})=0$. This is only possible if $t^2+1=t'^2+1=0$ and $p_{i+1}=\pm p_{i+2}$.
\end{proof}

\begin{lem} \label{lem:p1}
If $j_1$ is a P-joint and all other joints are R-joints, then $h_2||h_3$ and $h_4||h_5$.
\end{lem}

\begin{proof}
Assume that $j_1$ is the P-joint. It must be connected by a bond to at least one other joint.
By Lemma~\ref{lem:1}, this cannot be $j_2$ and $j_5$, so we may assume it is connected to $j_3$.
By Lemma~\ref{lem:1} again, $h_2$ and $h_3$ are parallel. Then the spherical projection $L'$ is a 
4-linkage with 2 equal axes $p_2$ and $p_3$. Hence $L'$ is degenerate, and the other two axes $p_4$ and
$p_5$ are equal too. 
\end{proof}

Note that Lemma~\ref{lem:p1} holds for all $j_i$ with appropriate indices.

\begin{thm} \label{thm:p}
Let $L$ be a 5-linkage with at least one P-joint and all other joints of type R.
Then the following two cases are possible.
\begin{enumerate}
\item Up to cyclic shift, $j_1$ is the only P-joint, $h_2||h_3$, and $h_4||h_5$;
	$L$ has mobility~1, and $t_2=\pm t_3$ and $t_4=\pm t_5$ is fulfilled on the configuration curve.
\item All axes of R-joints are parallel.
\end{enumerate}
\end{thm}

\begin{proof}
Using Lemma~\ref{lem:p1}, we get $h_2||h_3$, $h_4||h_5$. Without loss of generality, we assume $p_2=p_3$ and $p_4=p_5$;
if this is not true, it can be easily achieved by replacing $h_3$ by $-h_3$ or $h_5$ by $-h_5$.
Either all axes of R-joints are parallel or the axes of $h_3$ 
and $h_4$ are not parallel. In the second case, there is nothing left to show; let us assume that $h_3$
and $h_4$ are not parallel. 
The primal part of the closure equation is equivalent to the equality of the two rotations
 $$(t_2-p_2)(t_3-p_2) \equiv (t_5+p_5)(t_4+p_4).$$
Since the axes are distinct, both rotations must be the identity, which implies $t_2=-t_3$ and $t_4=-t_5$.
%
\end{proof}

\section{Construction of Linkages with Helical Joints}\label{construction-helical}

In this section we give a construction that produces mobile linkages with H-joints from linkages
with C-, P-, and R-joints. We illustrate the construction by several well-known examples and one
example which is new.

We start with a simple construction: take a linkage with $r$ C-joints that has
mobility at least $r+1$. For each C-joint $j_k$, impose the additional restriction $t_k=\cot(\frac{s_k}{2g_k})$
on its joint parameters $(s_k,t_k)$, where $g_k$ is a nonzero real constant. Any additional equation reduces
the mobility at most by 1, so we get a mobile linkage where every C-joint $j_k$ 
is replaced by an H-joint with pitch~$g_k$.

We can extend this simple construction using the observation that $\Q$-linear relations between the angles imply
algebraic relations between their tangents. For the general construction, which we call {\em screw carving},
we need the following ingredients.

\begin{enumerate}
\item a linkage $L$ with $m$ C-joints $j_{k_1},\dots,j_{k_m}$ and an undetermined number of R- and P-joints;
\item an irreducible analytic subspace $K_0$ of the configuration space of $L$;
\item an integer matrix $A$ with $m$ columns that annihilates the vector of analytic functions
	$(\alpha_{k_1},\dots,\alpha_{k_m})^t\in\C(K_0)^m$ such that $\cot(\frac{\alpha_k}{2})=t_k$;
\item an $m$-tuple $(g_{k_1},\dots,g_{k_m})$ of nonzero real numbers, so that $A$ also annihilates
	the vector of functions $(a_{k_1},\dots,a_{k_m})^t$, where $a_k:K_0\to\C$ is the function
	$(s_\ast,t_\ast)\mapsto \frac{s_k}{g_k}$.
\end{enumerate}

As before, the linkage $L'$ with H-joints instead of C-joints is obtained by imposing the additional restriction
$t_k=\cot(\frac{s_k}{2g_k})$ on its joint parameters $(s_k,t_k)$, for each C-joint $t_k$. To obtain linkages
with large mobility, the integer matrix $A$ should have the largest possible rank, which means that
all integral relations between the analytic angle functions are linear combination of matrix rows.
(In the next section, we will indeed always choose such matrices of maximal rank.) The empty
matrix with zero rows is allowed, then we just get the simple construction above.

\begin{lem} \label{lem:carving}
Let $d:=\dim(K_0)$ and $\ell:=\rank(A)$. Then the mobility of the linkage produced by screw carving is at least $d-m+\ell$.
\end{lem}

\begin{proof}
The subset $K'$ of $K_0$ that satisfies the additional restrictions $t_k=\cot(\frac{s_k}{2g_k})$ is contained
in the configuration space of $L'$. Since the codimension of an analytic subset is never bigger than the number
of defining equations, we see that $\dim(K')\ge d-m$. We claim that $K'$ can be defined (as a subset of $K'$)
by only $m-\ell$ equations.

Let $\alpha_{k_1},\dots,\alpha_{k_m}\in\C(K_0)$ be as above. The $\Q$-vector space generated by these $m$
functions has dimension at most $m-\ell$. Without loss of generality, we assume that 
$\{\alpha_{k_1},\dots,\alpha_{k_{m-\ell}}\}$ is a generating set. Any other $\alpha_k$ can be expressed as
a $\Q$-linear combination $$\alpha_k=q_1\alpha_{k_1}+\dots+q_{k_{m-\ell}}\alpha_{k_{m-\ell}},$$
with rational coefficients depending on the matrix $A$. But then we also have 
$$\frac{s_k}{g_k}=q_1\frac{s_{k_1}}{g_{k_1}}+\dots+q_{k_{m-\ell}}\frac{s_{k_{m-\ell}}}{g_{k_{m-\ell}}}.$$
It follows that the equations 
$t_{k_1}=\cot(\frac{s_{k_1}}{2g_{k_1}}),\dots,t_{k_{m-\ell}}=\cot(\frac{s_{k_{m-\ell}}}{2g_{k_{m-\ell}}})$
imply all other equations.
\end{proof}

\begin{example}
Let $L$ be a 4-linkage with 4 cylindrical joints with parallel axes. Its mobility is 4.
For all configurations
$(t_1=\cot(\frac{\alpha_1}{2}),s_1,\dots,t_4=\cot(\frac{\alpha_4}{2}),s_4)$, we have 
$\alpha_1+\alpha_2+\alpha_3+\alpha_4=0$ and $s_1+s_2+s_3+s_4=0$. So we take $K_0$ as the
full configuration set, $A$ as the $1\times 4$ matrix $(1,1,1,1)$, and $g_1=g_2=g_3=g_4$,
and apply screw carving. We obtain a 4-linkage with 4 helical joints and mobility 4-4+1=1.

Similarly, one can obtain an $n$-linkage with $n$ H-joints with parallel axes with mobility $n-3$, $n\geq 4$.
\end{example}

\begin{example} \label{ex:hhrrr}
Here is a variation of the previous example. Set
\[ h_1=\qk-\eps\qi, h_2=\qk+\eps\qi, h_3=h_5=\qk, h_4=\qk+2\eps\qj \]
and let $L$ be the CCRRR linkage with C-joint axes $h_1,h_2$ and R-joint axes $h_3,h_4,h_5$.
Its mobility is 3, and all configurations satisfy $s_1+s_2=0$. We define $K_0$ as the subvariety
defined by $\tan(17\arccot(t_1)-11\arccot(t_2))=0$ (this is a rational function in $t_1,t_2$).
Its dimension is 2. We set as the $1\times 2$ matrix $A=(1,1)$ and 
$g_1=\frac{1}{17}$, $g_2=\frac{-1}{11}$. By screw carving we get an HHRRR linkage with mobility 1.
Figure~\ref{fig:hhrrr} shows the trace of the joint $j_4$ when the link with the two H-joints $j_1,j_2$ is fixed.
\end{example}

\begin{figure}[htb]
  \centering
  \includegraphics[width=8cm]{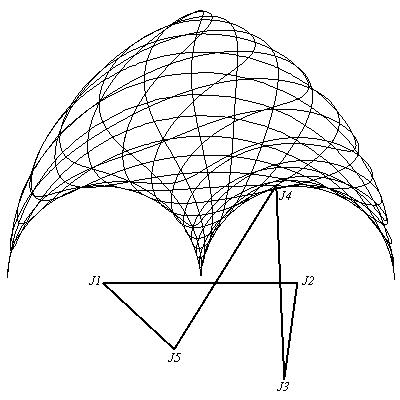}
  \caption{Planar projection of an HHRRR linkage with 5 parallel axes to the plane orthogonal to the axes
	(Example~\ref{ex:hhrrr}). The helical joints are at $j_1$ and $j_2$.
	The ratio of the pitches at the two helical joints $j_1,j_2$ is 11:17. The curve shown is the trace
	of the joint $j_4$. It is an algebraic curve of large degree.}
  \label{fig:hhrrr}
\end{figure}

\begin{example}
Let $h_1,h_2,h_3$ be lines. Reflecting them by the coordinate axes represented by $\qi$, we get
$h_4=\qi h_1\qi$, $h_5=\qi h_2\qi$, $h_6=\qi h_3\qi$. Let $L$ be the 6C-linkage with axes $h_1,\dots,h_6$.
The zero set of the closure equation
\[ (t_1-h_1)(1-\eps s_1h_1)\cdots(t_6-h_6)(1-\eps s_6h_6) \equiv 1 \]
has a component of dimension~4, given by the equations 
\[ t_1=t_4, t_2=t_5, t_3=t_6, s_1=s_4, s_2=s_5, s_3=s_6, x\qi+\qi x=0,\]
\[ \mbox{where } x=(t_1-h_1)(1-\eps s_1h_1)(t_2-h_2)(1-\eps s_2h_2)(t_3-h_3)(1-\eps s_3h_3) . \]
With $A=\begin{pmatrix} 1 & -1 & 0 & 0 & 0 & 0 \\ 0 & 0 & 1 & -1 & 0 & 0 \\ 0 & 0 & 0 & 0 & 1 & -1 \end{pmatrix}$
and $g_1=g_4$, $g_2=g_5$, $g_3=g_6$, the screw carving procedure gives a line symmetric 6H linkage
with mobility~1.

Similarly, one can construct a plane symmetric RHHRHH linkage with mobility~1. Both linkages are well-known,
see \cite{Baker:97}.
\end{example}

\begin{example} \label{example:new}
Let $h_1$, $h_2$, $h_3$ be lines with linear independent primal parts that do not intersect pairwise,
such that $o(h_1,h_2,h_3)=o(h_3,h_1,h_2)=0$. Let $L$ be the RRCRRC linkage with axes $h_1,h_2,h_3,h_2,h_1,h_3$.
The zero set of the closure equation
\[ (t_1-h_1)(t_2-h_2)(t_3-h_3)(1-\eps s_1h_3)(t_4-h_2)(t_5-h_1)(t_6-h_3)(1-\eps s_6h_3) \equiv 1 \]
has two components of dimension~2. The first is given by $t_1=-t_5$, $t_2=-t_4$, $t_3=t_6=\infty$, $s_3=s_6=0$;
this is a degenerate motion which does not separate the pairs of axes at joints $(j_1,j_5)$ and at $(j_2,j_4)$.
The equations of the second component $K_0$ can be computed by computer algebra. Two of them are $s_3=s_6$
and $t_3=t_6$; the remaining are more complicated. 
With $A=\begin{pmatrix} 1 & -1 \end{pmatrix}$
and $g_3=g_6$, the screw carving procedure gives an RRHRRH linkage with mobility~1.
In contrast to all families of mobile 6-linkages with H-joints that have been known up to now, this linkage
has no parallel axes or apparent geometric symmetries. A distinctive property is the existence of a starting
position with three pairs of coinciding axes.
\end{example}

\section{Classification of 5-linkages with Helical Joints}\label{classification-helical}

Now we show that the construction in Section~\ref{construction-helical} is complete, 
i.e., every linkage with helical joints can
be obtained in this way. The main idea is the application of a theorem by Ax to separate the transcendental 
part and the algebraic part of the closure equation. Then, we use the completeness result to classify
5-linkages with P-, R-, and H-linkages.

Let $L$ be a mobile linkage with $m$ helical joints; we assume that it is not degenerate,
and that all joints actually move. 
There is a natural candidate for ingredients of the screw carving constructions in order to produce $L$.

\begin{enumerate}[(i)]
\item We define the {\em cylindrical extension} $L'$ of $L$ be replacing all H-joints by C-joints.
	The configuration set $K$ of $L$ can be naturally embedded in the configuration $K'$ set of $L'$.
\item We take $K_0$ as the Zariski closure of $K$ in $K'$, i.e., the subset of $K'$ defined by all algebraic
	equations that hold for $K$.
\item We take $A$ as an integer matrix whose rows generate, as a $\Z$-module, the coefficient vectors
	of all integral linear equations that hold for the $m$ $\alpha$-parameters in $K$.
\item The $m$ nonzero numbers are defined as the $m$ pitches of $K$.
\end{enumerate}

Application of the screw carving constructions just re-installs the screw conditions that already existed
in $L$. But our construction includes a prediction on the mobility, and it is not clear if the mobility
of $L'$ is big enough to explain the mobility of $L$.

Since the use of number theoretic theorems is not usual in kinematics, we include the full statement
of the Theorem of Ax on the transcendence degrees of a function field with exponential functions.

\begin{thm} \label{thm:axorig}
Let $q,n$ be positive integers.
Let $f_1,\dots,f_q$ be analytic functions in some neighbourhood of $\C^m$ about the origin $o$ for which
$f_1-f_1(o),\dots,f_q-f_q(o)$ are $\Q$-linearly independent. Let $r$ be the rank of the Jacobi-matrix 
$\frac{\partial(f_1,\dots,f_q)}{\partial(z_1,\dots,z_m)}$. Then the transcendence degree of
$\C(f_1,\dots,f_q,e^{f_1},\dots,e^{f_q})$ is greater than or equal to $q+r$.
\end{thm}

\begin{proof}
This is \cite{ax}, Corollary 2.
\end{proof}

\begin{thm} \label{thm:ax}
Let $L$ be a linkage with $m$ helical joints. Let $K$ be an irreducible component of its configuration space
containing the initial configuration as a nonsingular point. 
Let $K_0$ be the Zariski closure of $K$ in the cylindrical extension of $L$.
Let $A$ be the integral matrix defined by the $\Z$-relations between the helical joint parameters. Then
\[ \dim(K_0) = \dim(K) + m - \rank(A) . \]
Consequently, $L$ can be obtained by screw carving from its cylindrical extension.
\end{thm}

\begin{proof}
By Lemma~\ref{lem:carving}, we have $\dim(K_0)\le \dim(K)+m-\rank(A)$, so it suffices to show the other inequality.
Let $d:=\dim(K)$, and $q:=m-\rank(A)$.
Let $\pi_H:K\to \C^m$ be the projection to the helical joint parameters $\alpha_{k_1},\dots,\alpha_{k_m}$,
and let $K_H\subset\C^m$ be its image. Similarly, we define $\pi_C:K_0\to\C^m\times(\P^1)^m$ as the
projection to the cylindrical joint parameters of $L'$ and $K_C$ as the image. There is a natural embedding
$K_H\hookrightarrow K_C$, the map $\pi_H$ is the restriction of $\pi_C$ along this embedding, 
and $K_C$ is the Zariski closure of $K_H$ in $\C^m\times(\P^1)^m$. 

Let $d_H:=\dim(K_H)$. Then there is an analytic isomorphism $\phi$ of a neighbourhood $U\subset\C^{d_H}$ of the origin $o$ 
mapping $o$ to the initial configuration. For $k=1,\dots,m$, let $\overline{\alpha_k}:U\to\C$ be the
projection to the joint parameter $\alpha_k$. They generate a $q$-dimensional $\Q$-vector space.
We may assume that $\overline{\alpha_{k_1}},\dots,\overline{\alpha_{k_q}}$ generate this vector space. 
The rank of the Jacobian of $\overline{\alpha_{k_1}},\dots,\overline{\alpha_{k_q}}$ is equal to the
rank of the Jacobian of all coordinate functions, which is equal to $d_H$. By Theorem~\ref{thm:axorig},
the field $\C(\overline{\alpha_{k_1}},\dots,\overline{\alpha_{k_q}},
e^{\overline{\alpha_{k_1}}},\dots,e^{\overline{\alpha_{k_q}}})$
has transcendence degree at least $d_H+q$ over $\C$. This field is $\C$-isomorphic to the function field of $K_C$,
by the isomorphism
\[ \overline{\alpha_{k_1}} \mapsto g_{k_1}^{-1}s_{k_1}, \dots, \overline{\alpha_{k_q}} \mapsto g_{k_q}^{-1}s_{k_q},
   e^{\overline{\alpha_{k_1}}} \mapsto \frac{1+\ci t_{k_1}}{1-\ci t_{k_1}}, \dots, 
   e^{\overline{\alpha_{k_q}}} \mapsto \frac{1+\ci t_{k_q}}{1-\ci t_{k_q}} . \]
Therefore $\dim(K_C)\ge d_H+q$. 

Let $E\subset K_C$ be the set of all points $x$ such that $\dim(\pi_C^{-1}(x)) > \dim(K_0)-\dim(K_C)$.
Since dimension is upper semicontinuous in the Zariski topology, $E$ is a proper algebraic subvariety of $K_C$.
Since $K_H$ is Zariski dense in $K_C$, it is not contained in $E$. Therefore the generic fibre of $\pi_H:K\to K_H$
has dimension $\dim(K_0)-\dim(K_C)$. Hence we have
\[ \dim(K_0) = \dim(K_C)+\dim(K)-\dim(K_H) \ge d_H+q+\dim(K)-d_H = \dim(K)+q . \]
\end{proof}

From the cylindrical extension $L_c$, we may construct families $F_p$ and $F_r$ of linkages by setting
either the rotation parameters or the translation parameters of the joints $j_{k_1},\dots,j_{k_q}$ to fixed values
$(\tau_{k_1},\dots,\tau_{k_q})$ respectively $(\sigma_{k_1},\dots,\sigma_{k_q})$,
where $\overline{\alpha_{k_1}},\dots,\overline{\alpha_{k_q}}$ generate the $\Q$-vector space of all
angle parameter functions at the helical joints. This imposes exactly $q$ additional equations to the configuration
set $K_0$, hence the mobility of any linkage in one of the two families is greater than or equal to the
mobility of $L$. In family $F_r$, every C-joint is replaced by an R-joint, and in family $F_p$, every
C-joint is replaced by a P-joint. 

We choose generic members for both families and call them $L_p$ and $L_r$.
The angles and orthogonal distances of neighbouring R-joints is constant for both families, and is equal
to the value of the corresponding parameter of $L$. The offsets do change, but in a transparent way:
if $j_2$ is an H-joint and $j_1,j_3$ are H- or R-joints in $L$, then the offset of the axes at the corresponding
axes in the family $F_r$ is a linear non-constant function in the family parameter $\sigma_2$. Similarly,
the angle between the directions of $j_1$ and $j_3$ from $j_2$ in spherical geometry change with the 
family parameter $\tau_2$. In particular, the generic family member $L_r$ has nonzero offset and the 
generic family member $L_p$ has nonzero angles at this place.

Linkages with helical joints are called degenerate if there are neighbouring R- or H-joints with equal axes, neighbouring
P-joints with equal directions, or an H-joint with a neighbouring P-joint in the direction of the axis of the
H-joint. In all these cases it is easily possible to simplify the pair of neighboring joints from RR to R, HH to H or C,
HR to C, PP to P, or HP to C. 

Assume that $L$ is nondegenerate.
Because the axes/directions of $L_r$ and $L_p$ are equal to the axes/directions of an instance of $L_c$ after
application of a motion in $K_0$, $L_r$ is also nondegenerate, and $L_p$ can only be degenerate if $L$ has
neighbouring H-joints with parallel axes. However, it may happen that some joints of $L_r$ or $L_p$ remain fixed
during the motion, even if this is not the case for $L$.

The existence of mobile linkages with only P- and R-joints with particular properties as a consequence of the 
existence of mobile linkages with H-joints allows to classify the 5-linkages with H-joints. 
We do that by constructing the linkages $L_r$ and $L_p$ and then compare with the classifications
in Section~\ref{non-helical} and Theorem~\ref{thm:p}. We also need the classification of 4-linkages, because
it may be that some joints of $L_r$ or $L_p$ are fixed. For convenience, we write here the facts on
4-and 5-linkages that are used below. For the first 4 facts, we assume that $L$ is a nondegenerate mobile linkage 
with R- and P-joints, such that every joint actually moves. Let $n$ be the number of joints of $L$.

\begin{enumerate}
\item If $n=4$ and $L$ has a P-joint, then all axes of R-joints are parallel 
	(this is a special case of Delassus' theorem \cite{Delassus}).
\item If $n=4$ and $L$ has no P-joints, then either all axes are parallel, 
	or no neighbouring pair of axes is parallel
	and all offsets are zero (see \cite{bhs}).
\item If $n=5$ and $L$ has a P-joint, then either all axes of R-joints are parallel, or $j_1$ is a P-joint,
	all other joints are of type R, and $h_2||h_3$ and $h_4||h_5$ (Theorem~\ref{thm:p}). 
\item If $n=5$ and $L$ has no P-joints, then either all axes are parallel, 
	or no neighbouring pair of axes is parallel
	and at least three of the five offsets are zero (see Section~\ref{sec:bonds}).
\item Any movable CRP linkage is degenerate, i.e., either the axis of the C-joint 
	and the axis of the R-joint coincide,
	or the axis of the C-joint is parallel to the direction of the P-joint.
	Here we leave the proof to the reader.
\end{enumerate}

Here is the classification of 5-linkages with joints of type R, P, or C, based on Theorem~\ref{thm:ax}
(compare also with Delassus's classification~\cite{Delassus} of 4-linkages with these three types of joints).

\begin{thm} \label{thm:h5}
Let $L$ be a non-degenerate mobile 5-linkage with R-, P-, and H-joints, with at least one H-joint, such that all
joints actually move. Up to cyclic permutation, the following cases are possible.
\begin{enumerate}
\item All axes of R- and H-joints are parallel. 
\item There is one P-joint $j_1$, all other joints are of type H or R, $h_2||h_3$ and $h_4||h_5$.
\end{enumerate}
\end{thm}

\begin{proof}
Let $r$ be the number of neighbouring blocks of equal axes of the spherical projection $L_s$.
The proof proceeds by case distinction on $r$. The cases $r=4$ and $r=5$ are split into two
subcases.

\begin{description}
\item[Case $r=1$] Then all axes of $L_s$ are equal, hence all axes of H- and R-joints of $L$ are parallel;
this is possibility~(1) of the theorem.

\item[Case $r=2$] Then each of the two blocks of R-joints in $L_s$ has at least two joints and at most three joints, 
because a single joint could not move. The linkage $L_r$ is movable and therefore
has at least four joints that actually move. In particular, it cannot happen that all axes of $L_r$ that
actually move are parallel. After removing the joints that remain fixed, $L_r$ still has two blocks 
of parallel axes. By comparing with Facts~3 and 4 above, it follows that $L_r$ is a
PRRRR linkage; if, say, $j_1$ is the prismatic joint, then $h_2||h_3$ and $h_4||h_5$.
This is the possibility~2 of the theorem.

\item[Case $r=3$] There is at least one group of joints of the spherical projection $L_s$ with only one R-joint. 
This joint cannot move.
Hence the corresponding H- or R-joint of $L$ does not move either, contradicting our assumption. So
this case is impossible.

\item[Case $r=4$] 
If $L$ has a P-joint, then $L_r$ is a mobile and nondegenerate PRRRR linkage without any parallel rotation axes. 
Such a linkage does not exist, hence $L$ has no $P$-joint. Then we have two parallel neighbouring axis of $L$.
Up to cyclic permutation, we may assume $h_1||h_2$, and the other directions of axes are not parallel. There is
no 4R or 5R linkage with exactly two parallel axes (see Fact~3). 
Therefore at least one of the two joints with parallel axes $h_1,h_2$ must be fixed in $L_r$.
Without loss of generality, we may assume that $L_r$ is a 4R linkage with axes $h_2,h_3,h_4,h_5$.
By Fact~2, all offsets of $L_r$ are zero.
In particular, $o(h_2,h_3,h_4)=o(h_3,h_4,h_5)=0$. Hence $j_3$ and $j_4$ are R-joints in $L$.
Hence $h_3$ and $h_4$ are axes of $R$-joints of $L_p$. On the other hand, they are not parallel to each
other and not parallel to the remaining axes. In any movable 4-linkage or 5-linkage with joints of type R and P,
at least one P, any revolute axis is in a block of at least two parallel axes (see Fact~1). It follows that
the two joints with axes $h_3$ and $h_4$ must be fixed in $L_p$. Hence $L_p$ has at most three joints that
actually move. This is not possible if $L_p$ has an R-joint that actually moves. Since $h_1||h_2$ and $h_5$
is not parallel to both, the joint of $L_p$ corresponding to $j_5$ is also fixed, and $L_p$ is a degenerate
linkage with two P-joints sharing the same direction. It follows that $j_1$ and $j_2$ are H-joints in $L$.
The type of joint $j_5$ may be either $H$ or $R$.

\begin{description} 
\item[Subcase~1] $j_5$ is an R-joint. The cylindrical extension $L_c$ has two C-joints with axes $h_1$ and $h_2$.
If the mobility of $L_c$ is 3 or higher, then we freeze one of the two C-joints, and
we get mobile CRRR linkage without parallel neighbouring axes, contradicting Fact~1. Therefore the mobility of $L_c$
is 2, and we must have a $\Q$-linear relation between the angle functions.
But in $L_r$, the joint corresponding to $j_1$ is fixed, and the joint corresponding to $j_2$ moves.
This is a contradiction.

\item[Subcase~2] $j_5$ is an H-joint. If the mobility of $L_c$ is 2, then we may argue as in Subcase~1:
the three angle functions at the C-joints of $L_c$ have to generate a one-dimensional $\Q$-vector space
of $L_c$. Especially, there is a $\Q$-linear relation between the angle functions at the joints corresponding
to $j_1$ and $j_2$. And in $L_r$, the joint corresponding to $j_1$ is fixed, and the joint corresponding to $j_2$ moves,
which gives a contradiction. Hence the mobility of $L_c$ is at least 3.
Now we freeze the R-joint with axes $h_3$ and the translation component at the joint corresponding to $j_2$.
We get a mobile CRRC linkage $L'$ with joint axes  $h_1,h_2,h_4,h_5$. Since $h_1||h_2$, and
$h_4$ and $h_5$ have different directions, the spherical projection of $L'$ is fixed at the two joints with
axes $h_4$ and $h_5$. 
Hence $L'$ is fixed at the joint with axes $h_4$ and the C-joint with axes $h_5$ can be replaced by a P-joint.
We may consider $L'$ as a CRP linkage with axes $h_1,h_2$, and translation direction $p_5$. By Fact~5, either
$h_1\equiv h_2$ or $h_1||h_5$. Neither is possible.
\end{description}

\item[Case $r=5$] Then $L_r$ is either a 5R-linkage or a 4R-linkage with an extra immobile R-joint. 
\begin{description}
\item[Subcase~1]
If $L_r$ is a 5R-linkage, then $L_r$ has at least three vanishing offsets, by Fact~4. Hence $L$ has at least one and at most two H-joints,
and $L_p$ has at least one at most two P-joints. Because $L_p$ has no parallel rotation axes, all its R-joints
must be fixed. Then $L_p$ must have two P-joints with equal directions. On the other hand, $L$ has no parallel
axes. This is a contradiction.

\item[Subcase~2]
Without loss of generality, let us say that the joint with axis $h_1$ is fixed in $L_r$. By Fact~2, all offsets of $L_r$
are zero. In particular, $o(h_2,h_3,h_4)=o(h_3,h_4,h_5)=0$. Hence $j_3$ and $j_4$ are R-joints. Again, $L_p$ has
no parallel rotation axes, so all its R-joints must be fixed. If $L_p$ has two moving P-joints, then they
would have to parallel, which is not possible. Hence $L_p$ has three moving P-joints with direction $p_5,p_1,p_2$.
Then these three vectors have to linear dependent, which is equivalent to saying that the angle between $p_5$
and $p_2$ from $p_1$ (in spherical geometry) is equal to $0$. On the other hand, $L_p$ has a nonzero angle
at every joint corresponding to an H-joint of $L$. This is a contradiction.
\end{description}
\end{description}
\end{proof}

\begin{example}
Here is an example that shows that the second case is indeed possible.

Let $h_2,h_3$ be parallel lines with a distance $a$ to each other. Let $h_4,h_5$ be another pair of parallel
lines, not parallel to the first pair, also with distance $a$ to each other. We assume that $p_2=-p_3$
and $p_4=-p_5$. Let $g_2,g_4\in\R$. For any $\alpha_2=\alpha_3\in\R$, the composed motion
\[ m_2m_3=(1-\eps g_2\alpha_2 p_2)(1-\tan\left(\frac{\alpha_2}{2}\right)h_2)
	(1-\eps g_2\alpha_3 p_3)(1-\tan\left(\frac{\alpha_3}{2}\right)h_3) \]
is a translation, where the translation vector lies on a circle in a plane orthogonal to $p_2$ with radius $a$.
Similarly, for $\alpha_4=\alpha_5\in\R$ the composed motion $m_4m_5$ is a translation with translation vector
on another circle with the same radius. We can choose a parametrisation such that the motion $m_2m_3m_4m_5$ is
a translation in a fixed direction ($p_2+p_4$ or $p_2-p_4$). Hence the linkage with P-joint in this direction
and H-joints with axes $h_2,h_3,h_4,h_5$ and pitches $g_2,g_2,g_4,g_4$ is mobile.
\end{example}

\section{Conclusion}

Using Ax's theorem and screw carving, it is possible to investigate mobility questions for arbitrary linkages
with helical joints. The classification of mobile closed 5-linkages with joints of type R, P, or H, given in this paper,
is just a first application of this reduction. 

A challenge for future research is the classification of mobile closed 6-linkages with helical joints. In contrast to 
the case of 5-linkages, reduction to linkages with joints of type R or P will not be enough, because our knowledge
of that linkages is still quite incomplete: even the classification of mobile closed 6R linkages is an open problem.
But there is a reason to believe that classifying 6-linkages with at least one H-joint is substantially easier than
the 6R case: the linkages constructed by cylindrical extension and fixing either the rotational or the translational
parameter have properties that could help the classification (for instance, generic offsets). Another possible attempt
would be to extend the theory of bonds to linkages with C-joints and to configuration sets where a fixed set of angle
functions satisfy ${\bf Z}$-linear equations.

\section*{Acknowledgements}
This research was supported by the Austrian Science Fund (FWF):
W1214-N15, project DK9.

\bibliographystyle{plain}
\bibliography{helical}

\hspace{0.5cm}

Johann Radon Institute for Computational and Applied Mathematics (RICAM), Austrian Academy of Sciences,  
Altenberger Str. 69, A-4040 Linz, Austria

\url{hamid.ahmadinezhad@oeaw.ac.at}

\url{zijia.li@oeaw.ac.at}

Research Institute for Symbolic Computation (RISC), Johannes Kepler University, 
Altenberger Str. 69, A-4040 Linz, Austria 

\url{josef.schicho@risc.jku.at}

\end{document}